\documentclass[3p,times]{elsarticle}

\usepackage{ecrc}

\volume{00}

\firstpage{1}

\journalname{Procedia Computer Science}

\runauth{}

\jid{procs}

\jnltitlelogo{Procedia Computer Science}

\CopyrightLine{2011}{Published by Elsevier Ltd.}

\usepackage{amssymb}

\usepackage[figuresright]{rotating}

\usepackage{afterpage}
\usepackage{makecell}
\usepackage{tabulary}
\usepackage{xcolor}
\usepackage{color}
\usepackage{colortbl}
\usepackage{prettyref}
\usepackage{framed}
\usepackage{threeparttable}
\usepackage{booktabs}       
\usepackage[utf8]{inputenc} 
\usepackage[T1]{fontenc}    
\usepackage{hyperref}       
\usepackage{url}            
\usepackage{booktabs}       
\usepackage{amsfonts}       
\usepackage{nicefrac}       
\usepackage{microtype}      
\usepackage{xcolor}         
\usepackage{amsthm}
\usepackage{amsmath}
\usepackage{amssymb}
\usepackage{algorithm}
\usepackage{algpseudocode}
\usepackage{graphicx}
\usepackage{subfigure}
\usepackage{booktabs} 
\usepackage{bm}
\usepackage{hyperref}
\usepackage{cleveref}
\usepackage{enumitem}
\usepackage{mathtools}
\usepackage{longtable}
\usepackage{graphicx}
\usepackage{natbib}
\usepackage{array}
\usepackage{lipsum}
\usepackage{multirow}
\usepackage{makecell}
\usepackage{xcolor}
\usepackage{tikz}
\usepackage{etoolbox}
\newcommand{\pref}[1]{\prettyref{#1}}

\newcommand{\savehyperref}[2]{\texorpdfstring{\hyperref[#1]{#2}}{#2}}
\newrefformat{eq}{\savehyperref{#1}{Eq. \textup{(\ref*{#1})}}}
\newrefformat{eqn}{\savehyperref{#1}{Eq.~\textup{(\ref*{#1})}}}
\newrefformat{lem}{\savehyperref{#1}{Lemma~\ref*{#1}}}
\newrefformat{assump}{\savehyperref{#1}{Assumption~\ref*{#1}}}
\newrefformat{defi}{\savehyperref{#1}{Definition~\ref*{#1}}}\newrefformat{tab}{\savehyperref{#1}{Table~\ref*{#1}}}
\newrefformat{lemma}{\savehyperref{#1}{Lemma~\ref*{#1}}}
\newrefformat{line}{\savehyperref{#1}{Line~\ref*{#1}}}
\newrefformat{thm}{\savehyperref{#1}{\textbf{Theorem}~\ref*{#1}}}
\newrefformat{corr}{\savehyperref{#1}{Corollary~\ref*{#1}}}
\newrefformat{cor}{\savehyperref{#1}{Corollary~\ref*{#1}}}
\newrefformat{sec}{\savehyperref{#1}{Section~\ref*{#1}}}
\newrefformat{app}{\savehyperref{#1}{\textbf{Appendix}~\ref*{#1}}}
\newrefformat{appendix}{\savehyperref{#1}{\textbf{\ref*{#1}}}}
\newrefformat{ex}{\savehyperref{#1}{Example~\ref*{#1}}}
\newrefformat{fig}{\savehyperref{#1}{Figure~\ref*{#1}}}
\newrefformat{alg}{\savehyperref{#1}{Algorithm~\ref*{#1}}}
\newrefformat{rem}{\savehyperref{#1}{Remark~\ref*{#1}}}
\newrefformat{conj}{\savehyperref{#1}{Conjecture~\ref*{#1}}}
\newrefformat{prop}{\savehyperref{#1}{Proposition~\ref*{#1}}}
\newrefformat{proto}{\savehyperref{#1}{Protocol~\ref*{#1}}}
\newrefformat{prob}{\savehyperref{#1}{Problem~\ref*{#1}}}
\newrefformat{claim}{\savehyperref{#1}{Claim~\ref*{#1}}}
\newrefformat{que}{\savehyperref{#1}{Question~\ref*{#1}}}
\newrefformat{op}{\savehyperref{#1}{Open Problem~\ref*{#1}}}
\newrefformat{fn}{\savehyperref{#1}{Footnote~\ref*{#1}}}
\newrefformat{property}{\savehyperref{#1}{Property~\ref*{#1}}}

\usepackage{xspace}
\usepackage{amsmath}
\usepackage{bbm}
\usepackage{mathrsfs}
\usepackage{bm}
\usepackage{makecell}
\usepackage{tabulary}

\newcommand{\calL}{\mathcal{L}}
\newcommand{\calG}{\mathcal{G}}

\newcommand{\calA}{\mathcal{A}}

\newcommand{\calS}{d}
\newcommand{\calM}{\mathcal{M}}

\newcommand{\calW}{\mathcal{W}}
\newcommand{\calU}{\mathcal{U}}
\newcommand{\calRO}{\mathcal{O}}
\newcommand{\R}{\mathbb{R}}
\newcommand{\calD}{\mathcal{D}}
\newcommand{\calO}{\mathcal{B}}

\newcommand{\E}{\mathbb{E}}

\newcommand{\wtilde}{\widetilde}

\newcommand{\calX}{\mathcal{X}}
\newcommand{\calY}{\mathcal{Y}}

\newcommand{\epspm}{\epsilon_{p,m}}

\newtheorem{assumption}{Assumption}[section]

\newcommand*{\circled}[1]{\lower.7ex\hbox{\tikz\draw (0pt, 0pt)%
    circle (.5em) node {\makebox[1em][c]{\small #1}};}}
\robustify{\circled}

\newtheorem{thm}{Theorem}[section]

\newtheorem{lem}[thm]{Lemma}

\usepackage{afterpage}
\usepackage{prettyref}
\usepackage{pgfplots}
\usepackage{caption}
\usepackage{amsmath,amssymb,amsfonts}
\usepgfplotslibrary{fillbetween}
\usepackage{framed}
\newrefformat{eq}{\savehyperref{#1}{Eq. \textup{(\ref*{#1})}}}
\newrefformat{eqn}{\savehyperref{#1}{Eq.~\textup{(\ref*{#1})}}}
\newrefformat{lem}{\savehyperref{#1}{\textbf{Lemma}~\ref*{#1}}}
\newrefformat{lemma}{\savehyperref{#1}{\textbf{Lemma}~\ref*{#1}}}
\newrefformat{defi}{\savehyperref{#1}{\textbf{Definition}~\ref*{#1}}}
\newrefformat{line}{\savehyperref{#1}{Line~\ref*{#1}}}
\newrefformat{thm}{\savehyperref{#1}{\textbf{Theorem}~\ref*{#1}}}
\newrefformat{corr}{\savehyperref{#1}{Corollary~\ref*{#1}}}
\newrefformat{cor}{\savehyperref{#1}{Corollary~\ref*{#1}}}
\newrefformat{sec}{\savehyperref{#1}{Section~\ref*{#1}}}
\newrefformat{table}{\savehyperref{#1}{Table~\ref*{#1}}}
\newrefformat{app}{\savehyperref{#1}{\textbf{Appendix}~\ref*{#1}}}
\newrefformat{assum}{\savehyperref{#1}{Assumption~\ref*{#1}}}
\newrefformat{ex}{\savehyperref{#1}{Example~\ref*{#1}}}
\newrefformat{fig}{\savehyperref{#1}{Figure~\ref*{#1}}}
\newrefformat{alg}{\savehyperref{#1}{Algorithm~\ref*{#1}}}
\newrefformat{rem}{\savehyperref{#1}{Remark~\ref*{#1}}}
\newrefformat{conj}{\savehyperref{#1}{Conjecture~\ref*{#1}}}
\newrefformat{prop}{\savehyperref{#1}{Proposition~\ref*{#1}}}
\newrefformat{proto}{\savehyperref{#1}{Protocol~\ref*{#1}}}
\newrefformat{prob}{\savehyperref{#1}{Problem~\ref*{#1}}}
\newrefformat{claim}{\savehyperref{#1}{Claim~\ref*{#1}}}
\newrefformat{que}{\savehyperref{#1}{Question~\ref*{#1}}}
\newrefformat{op}{\savehyperref{#1}{Open Problem~\ref*{#1}}}
\newrefformat{fn}{\savehyperref{#1}{Footnote~\ref*{#1}}}

\usepackage{xspace}
\usepackage{amsmath}
\usepackage{bbm}
\usepackage{mathrsfs}
\usepackage{amssymb}
\usepackage{bm}
\usepackage{makecell}
\usepackage{tabulary}

\newcommand{\protector}{}

\newtheorem{defi}{Definition}[section]

\newcommand{\RcalD}{\mathcal{D}}

\newtheorem*{remark}{Remark}

\begin{document}

\begin{frontmatter}



\renewcommand{\thefootnote}{\fnsymbol{footnote}}

\dochead{}

\title{Bridging Privacy and Robustness for Trustworthy Machine Learning}


\author[label1]{Xiaojin Zhang
\footnote{National Engineering Research Center for Big Data Technology and System, Services Computing Technology and System Lab, Cluster and Grid Computing Lab, School of Computer Science and Technology, Huazhong University of Science and Technology, Wuhan, 430074, China}}
\address[label1]{Huazhong University of Science and Technology, China}
\author[label1]{Wei Chen}

\begin{abstract}
      The widespread adoption of machine learning necessitates robust privacy protection alongside algorithmic resilience. While Local Differential Privacy (LDP) provides foundational guarantees, sophisticated adversaries with prior knowledge demand more nuanced Bayesian privacy notions, such as Maximum Bayesian Privacy (MBP) and Average Bayesian Privacy (ABP), first introduced by \cite{zhang2022no}. Concurrently, machine learning systems require inherent robustness against data perturbations and adversarial manipulations. This paper systematically investigates the intricate theoretical relationships among LDP, MBP, and ABP. Crucially, we bridge these privacy concepts with algorithmic robustness, particularly within the Probably Approximately Correct (PAC) learning framework. Our work demonstrates that privacy-preserving mechanisms inherently confer PAC robustness. We present key theoretical results, including the formalization of the established LDP-MBP relationship, novel bounds between MBP and ABP, and a proof demonstrating PAC robustness from MBP. Furthermore, we establish a novel theoretical relationship quantifying how privacy leakage directly influences an algorithm's input robustness. These results provide a unified theoretical framework for understanding and optimizing the privacy-robustness trade-off, paving the way for the development of more secure, trustworthy, and resilient machine learning systems.
\end{abstract}


\end{frontmatter}



\section{Introduction}
The advent of machine learning (ML) has ushered in transformative changes across diverse domains—from healthcare diagnostics to financial modeling and autonomous systems—driven by its unprecedented capability to extract meaningful insights from complex datasets \citep{ahmed2022artificial, habehh2021machine}. However, this pervasive reliance on data, particularly sensitive information, raises critical concerns regarding privacy and security \citep{shokri2017membership}. Simultaneously, as ML models are increasingly deployed in dynamic real-world environments, their reliability under input inaccuracies or noise—quantified as \emph{algorithmic robustness}—has become paramount \citep{sehwag2019analyzing, xu2012robustness}. Building trustworthy, deployable AI systems necessitates addressing both challenges holistically.

The traditional paradigm of Local Differential Privacy (LDP) provides a cornerstone for privacy preservation, bounding the influence of any individual data point on computational outputs to obscure its presence \citep{dwork2008differential}. Particularly valuable in distributed settings where users perturb data locally, LDP offers strong theoretical guarantees. Yet sophisticated adversaries may exploit additional attack vectors. Recent research introduces Bayesian privacy frameworks—Average Bayesian Privacy (ABP) and Maximum Bayesian Privacy (MBP)—that model adversaries with prior knowledge and Bayesian inference capabilities \citep{zhang2022no}. These frameworks deliver more comprehensive protection by accounting for realistic adversarial strategies. These critical domains have largely evolved independently of algorithmic robustness research, creating substantial theoretical and practical gaps: (i) a lack of unified principles connecting nuanced privacy guarantees with robustness, and (ii) mechanisms that optimize one property at the expense of the other, while yielding suboptimal trade-offs. Bridging these gaps is essential for establishing trustworthy next-generation AI foundations.

Our work addresses these challenges through fundamental theoretical innovations that establish quantifiable synergies between privacy preservation and algorithmic robustness. Our contributions are threefold:
\begin{itemize}
    \item {First}, we derive the inaugural quantitative relationship between MBP and ABP (Theorem~\ref{thm:relation_avg_and_max_privacy}). This result precisely bounds expected privacy leakage (ABP) as a function of worst-case protection (MBP), demonstrating that strengthening MBP inherently enhances ABP—providing crucial mechanism-design guidance.
    \item {Second}, we formally establish that privacy-preserving mechanisms inherently confer Probably Approximately Correct (PAC) robustness (Theorem~\ref{thm: second_main_result_mt}). This profound equivalence proves an $(\epsilon,\beta,\alpha)$-MBP mechanism guarantees $\alpha$-PAC-Robustness, with quantifiable confidence degradation $\propto\epsilon\beta$, thereby bridging privacy and stability domains.
    \item {Third}, we reveal a novel mathematical duality between privacy leakage and input robustness (Theorem~\ref{thm: first_main_result}). Our analysis demonstrates that intentional parameter distortion ($\Delta$) from privacy mechanisms directly amplifies adversarial resilience ($\alpha$), enabling principled joint optimization via strategic noise injection.
\end{itemize}

Collectively, these contributions establish a theoretical foundation for developing ML systems that simultaneously resist privacy reconstruction attacks and maintain performance under data perturbations. Our framework delivers: (1) quantitative conversions among Bayesian privacy metrics, (2) formal privacy-to-robustness implications, and (3) pathways for unified optimization.

The remainder of this paper is organized as follows. Section \ref{sec:related_work} reviews related work on privacy measurements, relationships among privacy metrics, intersections of privacy and PAC learning, robustness measurements, and the relationship between privacy and robustness. In Section \ref{sec: framework}, we introduce our privacy attack and defense framework. In Section \ref{sec: connection_privacy_and_pac_robustness}, we define metrics for privacy and robustness, and establish theoretical connections between privacy notions and PAC robustness. We also introduce practical estimation approaches and analyses for MBP and ABP in \pref{appendix:measure_mbp} and \pref{appendix:measure_abp}. Finally, we conclude the paper and discuss future research directions.

For convenience, we list the important notations used throughout this paper in Table \ref{table: notation}.
\begin{longtable}{cc}
\caption{Table of Notation} \label{table: notation} \\
\toprule
Notation & Meaning \\
\midrule
$\epsilon_p$ & privacy leakage of the protector \\
$\calL$ & loss function \\
$\Delta$ & distortion extent \\
$\breve W$ & the original model parameter \\
$W$ & the released model parameter \\
$\calX$ & a set of instances \\
$\calY$ & a set of labels \\
$P$ & distribution over instances $\calX$ \\
$\calS$ & dataset \\
TV(·) & total variation distance \\
$\mathcal{A}$(·) & original model \\
$\mu$/$\mu(P)$ & statistic of distribution $P$ \\
$\mathcal{D}'$ & another dataset different from $\mathcal{D}$ \\
$n$ & size of dataset $\calD$\\
$L$ & Lipschitz Constant \\
$I$ & number of optimization rounds \\
$F^\mathcal{A}$ & belief distribution after observing protected information\\
$F^\mathcal{B}$ & belief distribution before observing protected information\\
$F^\mathcal{O}$ & belief distribution after observing unprotected information\\
$f^\mathcal{O}_D(d),f^\mathcal{B}_D(d),f^\mathcal{A}_D(d)$ & density function of $F^\mathcal{O},F^\mathcal{B},F^\mathcal{A}$\\
$d_H$(·,·) & Hamming distance \\
$d^{(m)}$ & m-th data in dataset\\
$\xi$ & maximum bayesian privacy constant\\
$f_{D|W}(d|w)$ & posterior probability density function \\
$f_D(d)$ & prior distribution\\
$f_{W|D}(w|d)$ & likelihood function based on observed parameter $W$\\
$\kappa_1$ & $f_{D}^\mathcal{O}(d)$ \\
$\hat{\kappa_1}$ & estimator of $\kappa_1$\\
$\kappa_2$ & $f_{D}^{\mathcal{B}}(d)$\\
$\hat{\kappa_2}$ & estimator of $\kappa_2$\\
$\kappa_3$ & MBP, i.e., $\xi$\\
$\hat{\kappa_3}$ & estimator of $\kappa_3$\\
\bottomrule
\end{longtable}

\section{Related Work}
\label{sec:related_work}

\subsection{Privacy Measurements}
\label{sec:related_DP}

The rigorous formalization of differential privacy (DP) by \cite{dwork2006calibrating} established a gold standard for privacy-preserving data analysis, catalyzing extensive research into privacy quantification. Subsequent works have expanded this foundation along several dimensions: computational feasibility, distributional awareness, and practical applicability. \cite{abadi2016deep}'s pioneering work demonstrated how deep neural networks could be trained under constrained privacy budgets by integrating DP-SGD with optimization techniques like Momentum \cite{rumelhart1986learning} and AdaGrad \cite{duchi2011adaptive}. This approach balanced the tension between model utility and privacy preservation, enabling privacy-sensitive deep learning applications.

Bayesian perspectives introduced significant advancements in privacy measurement. \cite{BDP-icml20} developed Bayesian Differential Privacy (BDP), incorporating data distribution considerations into the privacy calculus while establishing formal connections to traditional DP guarantees. Information-theoretic approaches emerged as powerful alternatives, with \cite{rassouli2019optimal} employing total variation distance for privacy measurement while assessing utility through minimum mean-square error, mutual information, and error probability. Similarly, \cite{Blauw2017metrics} proposed entropy-based metrics for quantifying online user privacy, acknowledging the contextual nature of privacy risks.

The practical challenges of achieving pure DP motivated alternative quantification frameworks. \cite{eilat2020bayesian} formalized privacy loss through relative entropy (Kullback-Leibler divergence), measuring information leakage via posterior belief updates. While \cite{foulds2016theory} identified limitations in posterior sampling for privacy protection, they demonstrated the effectiveness of Laplace mechanisms in Bayesian inference contexts. Social network privacy inspired novel metrics like Audience and Reachability \cite{Alemany2019sharing}, while \cite{du2012privacy} pioneered minimax formulations for information leakage using KL divergence.

Our work contributes to this landscape by adopting Jensen-Shannon (JS) divergence for privacy measurement, following \cite{zhang2022no,zhang2022trading}. This choice is theoretically grounded in JS divergence's symmetry and triangle inequality properties \cite{nielsen2019jensen, endres2003new}, which enable more nuanced quantification of privacy-utility trade-offs. This approach provides distinct advantages for analyzing the privacy-robustness relationships central to our framework.

\subsection{Relation between Distinct Privacy Measurements}
\label{sec:related_relation_PM_RM}

The evolution of differential privacy has spurred development of nuanced privacy metrics with complex interrelationships, as summarized in Table \ref{table: relation_privacy_robustness}. Several significant conceptual frameworks have emerged:

\begin{table}[h!]
\begin{center}
\renewcommand\arraystretch{1.5}
    \caption{Taxonomy of privacy metrics and their theoretical relationships}
    \begin{threeparttable}
    \begin{tabular}{lll}
    \toprule
    \textbf{Reference} & \textbf{Privacy Metric}\tnote{1} & \textbf{Relation To Other Metrics}\\
    \midrule
    \cite{dwork2016concentrated} & Concentrated DP & DP $\Rightarrow$ Concentrated DP \\
    \cite{mironov2017renyi} & Rényi-DP & Rényi-DP $\Rightarrow$ $(\epsilon,\delta)$-DP\\
    \cite{mironov2009computational}\tnote{2} & IND-CDP, SIM-CDP & SIM-CDP $\Rightarrow$ SIM$_{\forall\exists}$-CDP\tnote{3} $\Rightarrow$ IND-CDP\\
    \cite{dong2019gaussian} & $f$-DP & Gaussian DP\tnote{4} $\Rightarrow$ Rényi-DP \\
    \cite{blum2013learning} & Distributional Privacy & Distributional Privacy $\Rightarrow$ DP\\
    \cite{triastcyn2020bayesian} & Bayesian-DP & Bayesian DP $\Rightarrow$ DP \\
    \cite{cuff2016differential} & Mutual-Information DP & MI-DP $\Leftrightarrow$ DP (asymptotically)\\
    Our Work & $\xi$-MBP, $\epsilon$-ABP & MBP $\Rightarrow$ ABP\\
    \bottomrule
    \end{tabular}
    \begin{tablenotes}
        \footnotesize
        \item[1] All "DP"s denote differential privacy variants
        \item[2] Introduces Computational Differential Privacy (CDP)
        \item[3] SIM$_{\forall\exists}$-CDP extends SIM-CDP
        \item[4] Gaussian DP is a specific $f$-DP instance
      \end{tablenotes}
  \end{threeparttable}    
  \label{table: relation_privacy_robustness}
\end{center}
\end{table}

\textbf{Refinements of DP:} \cite{dwork2016concentrated} introduced Concentrated DP, focusing on mean privacy leakage. \cite{mironov2017renyi}'s Rényi-DP unified various DP relaxations through $\alpha$-divergences. Computational perspectives emerged with \cite{mironov2009computational}'s IND-CDP and SIM-CDP, which formalized privacy against bounded adversaries. \cite{dong2019gaussian} advanced $f$-DP and Gaussian DP, providing tighter composition bounds.

\textbf{Distribution-Aware Privacy:} \cite{blum2013learning} proposed Distributional Privacy, offering stronger protection than traditional DP through synthetic data generation. \cite{triastcyn2020bayesian} introduced Bayesian DP that incorporates dataset distributions while maintaining DP compatibility. Local DP solutions addressed scenarios without trusted third parties \cite{cormode2018privacy}, with convergence-privacy trade-offs analyzed by \cite{duchi2013local} and practical mechanisms developed by \cite{ding2017collecting}.

\textbf{Taxonomic Frameworks:} Several studies systematically compared privacy notions. \cite{cuff2016differential} established asymptotic equivalence between Mutual-Information DP and traditional DP. \cite{meiser2018approximate} clarified distinctions between approximate and probabilistic DP. \cite{desfontaines2019sok} classified privacy definitions into seven dimensions with partial ordering. \cite{guingona2023comparing} compared approximate and probabilistic DP, while \cite{unsal2023information} connected DP to information-theoretic measures.

Our contribution introduces Maximum Bayesian Privacy (MBP) and Average Bayesian Privacy (ABP) as novel Bayesian privacy metrics. Crucially, we integrate the established equivalence between Local DP and Maximum Bayesian Privacy from \cite{jiang2018context} into our framework, while demonstrating new connections to robustness. Specifically:
\begin{itemize}
    \item We formalize the known relationship: $\xi$-LDP $\Rightarrow$ $\xi$-MBP $\Rightarrow$ $2\xi$-LDP \cite{jiang2018context}
    \item We prove a novel quantitative relationship: MBP $\Rightarrow$ ABP
\end{itemize}
This synthesis bridges Bayesian privacy frameworks with traditional DP concepts, providing new perspectives for adversarial settings where prior knowledge affects privacy risks. Our taxonomy complements existing work by: (1) Introducing Bayesian metrics sensitive to prior distributions, (2) Establishing quantitative conversions between Bayesian privacy notions (MBP and ABP), and (3) Enabling principled algorithm design for contexts with informed adversaries.

\subsection{Intersections of Privacy and PAC Learning}
\label{sec:privacy_pac}

The convergence of privacy guarantees and Probably Approximately Correct (PAC) learning frameworks has generated profound theoretical insights. Seminal work by \cite{kasiviswanathan2011can} established that most concept classes learnable non-privately remain learnable under DP constraints, albeit with polynomial sample complexity increases. Their discovery of equivalence between local private learning and statistical query models demonstrated privacy need not fundamentally limit learnability.

Computational aspects were explored by \cite{bun2020computational}, who delineated connections between private PAC learning and online learning paradigms. \cite{feldman2014sample} revealed a profound equivalence between differentially private sample complexity and randomized one-way communication complexity, bridging two distinct theoretical domains. Practical implementations advanced with \cite{agarwal2017price}'s differentially private online convex optimization that maintained near-optimal regret bounds.

Our work extends this research trajectory by investigating how privacy-preserving mechanisms inherently confer PAC robustness. While prior studies focused primarily on learnability and sample complexity, we establish formal connections between Bayesian privacy frameworks (MBP/ABP) and resilience to data perturbations. This provides new theoretical foundations for developing algorithms that simultaneously ensure: (1) Provable privacy against reconstruction attacks, and (2) Certified robustness against adversarial inputs--dual objectives previously addressed separately.

\subsection{Robustness Measurements}
\label{sec:related_RM}

Robustness quantification has evolved along two primary dimensions: adversarial resistance and distributional stability. \cite{hendrycks2019benchmarking} introduced the critical distinction between corruption robustness (average-case performance under natural corruptions) and perturbation robustness (targeted adversarial resistance), framing robustness as a multi-faceted requirement. For neural networks, \cite{bastani2016measuring} formalized point-wise robustness through adversarial frequency and severity metrics, enabling granular vulnerability assessment.

Theoretical foundations advanced with \cite{zhang2019theoretically}'s decomposition of robust classification error into natural and boundary components, elucidating accuracy-robustness trade-offs. Practical certification methods emerged through \cite{cohen2019certified}'s randomized smoothing approach, which scaled certifiable robustness to ImageNet-scale problems. These developments collectively established robustness measurement as a rigorous subfield with both theoretical depth and practical applicability.

Our robustness framework builds upon these foundations while introducing PAC-formalized robustness guarantees that explicitly connect to privacy-preserving mechanisms. This dual perspective enables unified analysis of algorithms providing both privacy and robustness certificates.

\subsection{Relation between Robustness and Privacy}
\label{sec:related_relation_PR}

The intrinsic connection between algorithmic privacy and robustness has emerged as a fundamental research frontier, with key relationships summarized in Table \ref{table: relation_privacy_robustness}. Early connections were identified by \cite{Dwork2009statistics}, who demonstrated how robust statistical estimators naturally enable differentially private estimation through the Propose-Test-Release paradigm. \cite{pinot2019unified} later established formal equivalences between Rényi-DP and distributional robustness through $D_\lambda$ divergences.

Three significant research threads have developed:
\begin{enumerate}
    \item \textbf{Privacy-to-Robustness:} \cite{lecuyer2019certified}'s PixelDP framework demonstrated how $(\epsilon,\delta)$-DP mechanisms can provide certified robustness guarantees. \cite{bansal2020extending} introduced robust privacy as a generalized indistinguishability criterion.
    
    \item \textbf{Robustness-to-Privacy:} \cite{asi2023robustness} showed robust estimators can be transformed into private estimators via $\rho$-smooth inverse sensitivity mechanisms. \cite{hopkins2023robustness} established black-box reductions from robustness to privacy with optimal trade-offs.
    
    \item \textbf{Federated Settings:} \cite{pillutla2022robust} developed median-based aggregation providing simultaneous privacy, robustness, and communication efficiency in federated learning, though with communication trade-offs.
\end{enumerate}

Our work advances all three directions through several key contributions:
\begin{itemize}
    \item We establish bidirectional implications: $\text{privacy} \Leftrightarrow \text{robustness}$ for MBP and $\text{robustness} \Rightarrow \text{privacy}$ for ABP
    \item Introduce PAC-formalized robustness with explicit parameter relationships to privacy guarantees
    \item Develop quantitative conversion bounds between privacy and robustness parameters
    \item Provide the first unified framework connecting Bayesian privacy with PAC robustness
\end{itemize}
These contributions offer new theoretical tools for co-designing algorithms that simultaneously optimize privacy and robustness--a critical requirement for trustworthy AI systems in sensitive domains.

\begin{table}[h!]
\begin{center}
\renewcommand\arraystretch{1.5}
\scriptsize
    \caption{Theoretical relationships between privacy and robustness frameworks}
    \begin{threeparttable}
    \begin{tabular}{p{0.2\textwidth}p{0.2\textwidth}p{0.2\textwidth}p{0.3\textwidth}}
    \toprule
    \textbf{Reference} & \textbf{Privacy Metric} & \textbf{Robustness Metric} & \textbf{Theoretical Relation}\\\midrule
    \cite{Dwork2009statistics} & $(\epsilon,\delta)$-DP & M-estimators & Propose-Test-Release \\
    & & Robust covariance & paradigm \\
    & & Robust regression & \\
    \midrule
    \cite{lecuyer2019certified} & $(\epsilon,\delta)$-PixelDP & Certifiable robustness & \((\epsilon,\delta)\)-DP + conditions \\
    & & & \(\Rightarrow\) robustness \\
    \midrule
    \cite{pinot2019unified} & Rényi-DP & $D_{\mathcal{P}(\gamma)}$-$(\alpha,\epsilon,\gamma)$ robust & Equivalence: \\
    & & & $D_{\mathcal{P}(\gamma)}$-robust \(\Leftrightarrow\) Rényi-DP \\
    \midrule
    \cite{bansal2020extending} & Flexible Accuracy & Laplace Mechanism & Duality: \\
    & Robust Privacy & $M_{Lap}^{f,b}$ & robustness \(\Leftrightarrow\) privacy \\
    \midrule
    \cite{pillutla2022robust} & Secure aggregation & Median breakdown point & Compatibility requires \\
    & & & increased communication\tnote{1} \\
    \midrule
    \cite{hopkins2023robustness} & $(\epsilon,\delta)$-DP & $(\eta,\beta,\alpha)$-robust & Robustness \(\Leftrightarrow\) privacy \\
    \midrule
    \cite{asi2023robustness} & $\epsilon$-DP & $(\tau,\beta,\alpha)$-robust & Privacy $\Rightarrow$ robustness\tnote{2} \\
    & $(\epsilon,\delta)$-DP & $(\tau,\beta,\alpha)$-robust & Robustness$\Rightarrow$privacy\tnote{3} \\
    \midrule
    Our work & $\xi$-MBP & $(\tau,\beta,\alpha)$-robust & Bidirectional: \\
    & $\epsilon$-ABP & $(\tau,\beta,\alpha)$-robust & privacy $\Leftrightarrow$ robustness (MBP) \\
    & & & robustness $\Rightarrow$ privacy (ABP) \\
    \bottomrule
    \end{tabular}
    \begin{tablenotes}
        \footnotesize
        \item[1] No quantitative relationship established
        \item[2] When corruption ratio $\tau \approx \log(n)/n\epsilon$
        \item[3] When $\tau \geq \frac{2(d \cdot \log(\frac{R}{\alpha_0}+1)+\log \frac{2}{\beta})}{n\epsilon}$
      \end{tablenotes}
  \end{threeparttable}    
  \label{table: relation_privacy_robustness}
\end{center}
\end{table}

\section{Privacy Attack and Defense Framework}\label{sec: framework}

This section formalizes the adversarial ecosystem inherent in machine learning data privacy. We systematically establish a comprehensive threat model characterizing adversarial objectives, capabilities, knowledge structures, and methodologies, while concurrently delineating defensive countermeasures for sensitive information protection. This dual-aspect framework builds upon extensive prior research \cite{zhang2023towards, zhang2023game, kang2022framework, zhang2023theoretically}.

\subsection{Privacy Attack Model}\label{subsec:threat_model}

The threat model rigorously conceptualizes vulnerabilities compromising client data privacy through four adversarial dimensions:

\paragraph{Attacker’s Objective}
Our framework models a protector operating algorithm $\mathcal{M}(\cdot): \mathcal{D} \rightarrow \breve{W}$ and an adversarial recipient aiming to infer confidential aspects of $\mathcal{D}$ from $\breve{W}$ while maintaining functional utility. Formally, the adversary seeks to refine posterior beliefs $\mathbb{P}(\mathcal{D}|\breve{W})$ through analytical scrutiny of the observable output.

\paragraph{Attacker’s Capability}
We assume a semi-honest ("honest-but-curious") adversary who faithfully executes protocol specifications while illicitly reconstructing $\mathcal{D}$ through statistical analysis of $\breve{W}$, maintaining operational integrity throughout the privacy attack.

\paragraph{Attacker's Knowledge}
The adversary possesses access to the sanitized output $W$ (the transformed version of $\breve{W}$), potential prior knowledge of $\mathcal{D}$ obtained through historical breaches, and understanding of the utility constraints governing analytical operations.

\paragraph{Attacker's Method}
Given observation $W$, the attacker implements Bayesian inference to maximize posterior belief:
\begin{equation}
    \arg\max_{d} \log f_{D|W}(d|w) = \arg\max_{d} \left[ \underbrace{\log f_{W|D}(w|d)}_{\text{log-likelihood}} + \underbrace{\log f_{D}(d)}_{\text{log-prior}} \right]
\end{equation}
where $f_{D|W}$, $f_{W|D}$, and $f_{D}$ denote the posterior density, likelihood function, and prior distribution respectively.

\subsection{Privacy Protection Model}

We articulate defensive measures against unauthorized disclosure through four protector dimensions:

\paragraph{Protector’s Objective}
The protector engineers a stochastic privacy-preserving mechanism $\mathcal{P}: \breve{W} \rightarrow W$ to generate a carefully degraded information derivative $W$ that satisfies formal privacy guarantees under rigorous information-theoretic metrics.

\paragraph{Protector's Capability}
The protector possesses operational capacity to generate and transmit $W$ instead of $\breve{W}$, obscuring $\mathcal{D}$ through controlled perturbation while contending with potential information leakage from historical disclosures.

\paragraph{Protector's Knowledge}
The protector maintains complete knowledge of the sensitive dataset $\mathcal{D}$, pristine algorithm output $\breve{W}$, computational semantics of $\mathcal{M}(\cdot)$, operational parameters of $\mathcal{P}(\cdot)$, and crucially, the adversary's utility constraints.

\paragraph{Protector's Method}
To safeguard $\mathcal{D}$-confidentiality, the protector deploys a randomized mechanism $\mathcal{P}(\breve{W})$ that produces $W$ satisfying dual objectives:
\begin{align*}
    \mathbb{E}[\mathcal{U}(W)] & \geq \tau & &\text{(Utility preservation)} \\
    \mathcal{I}(\mathcal{D};W) & \leq \epsilon & &\text{(Privacy guarantee)}
\end{align*}
where $\mathcal{U}$ represents the utility function, $\tau$ denotes the minimum utility threshold, $\mathcal{I}$ quantifies privacy leakage, and $\epsilon$ establishes the privacy budget. This framework achieves a provable trade-off between data utility and confidentiality protection, implemented through optimized perturbation algorithms \cite{zhang2023theoretically}.

\section{Connecting Privacy Measurements and PAC Robustness in Machine Learning}\label{sec: connection_privacy_and_pac_robustness}

This section establishes the theoretical foundations connecting privacy frameworks with robustness guarantees in machine learning. We systematically investigate the mathematical relationships between Maximum Bayesian Privacy (MBP), Average Bayesian Privacy (ABP), and Local Differential Privacy (LDP), while demonstrating how these privacy notions fundamentally relate to Probably Approximately Correct (PAC) robustness. Through rigorous definitions and theorems, we prove equivalence relations between privacy frameworks under specific conditions and derive quantitative bounds for converting between privacy metrics. Crucially, we establish that privacy-preserving mechanisms inherently confer robustness properties, providing a unified theoretical framework for designing algorithms that simultaneously ensure data confidentiality and resilience to perturbations. These results bridge the gap between privacy and robustness research communities, offering new principles for developing trustworthy learning systems.

\subsection{Measurements for Bayesian Privacy}

This subsection formalizes the Bayesian privacy framework for quantifying inference risks in data publishing. We introduce three complementary privacy metrics--Maximum Bayesian Privacy (MBP), Average Bayesian Privacy (ABP), and their relationship to Local Differential Privacy (LDP)--that measure protection against training data reconstruction attacks. These metrics capture different aspects of privacy risk: MBP quantifies worst-case vulnerability, ABP measures expected privacy loss, and LDP provides distributional guarantees. The definitions establish a rigorous mathematical foundation for evaluating privacy-preserving mechanisms, enabling precise comparison between protection techniques and formal analysis of privacy-utility trade-offs.

\begin{defi}[Average Bayesian Privacy]\label{defi: average_privacy_JSD}
Let $\calM$ represent the protection mechanism mapping private information $D$ to released information $W$, and $\calG$ represent the attacker's inference mechanism. The privacy leakage $\epsilon_{p,a}$ is quantified using the square root of Jensen-Shannon divergence between:
\begin{align}
\epsilon_{p,a}(\calM, \calG) = \sqrt{{\text{JS}}(F^{\calA} \parallel F^{\calO})}
\end{align}
where $F^{\calA}$ denotes the attacker's posterior belief about $D$ after observing $W$, $F^{\calO}$ represents the prior belief without observation, and $F^{\calM} = \frac{1}{2}(F^{\calA} + F^{\calO})$ serves as the reference distribution. The mechanism satisfies $\epsilon$-ABP when $\epsilon_{p,a} \leq \epsilon$.

\textbf{Remark:}\\
For discrete $D$, the JS divergence decomposes as:
\begin{align*}
{\text{JS}}(F^{\calA} \parallel F^{\calO}) = \frac{1}{2}\left[\text{KL}\left(F^{\calA} \parallel F^{\calM}\right) + \text{KL}\left(F^{\calO} \parallel F^{\calM}\right)\right]
\end{align*}
\end{defi}

This definition provides a symmetric, bounded measure of privacy leakage that overcomes limitations of KL divergence. By quantifying the information gain between prior and posterior beliefs, it captures the effectiveness of privacy protection against Bayesian inference attacks. The JS divergence formulation ensures the metric is well-defined, symmetric, and satisfies the triangle inequality, enabling rigorous composition and comparison of privacy mechanisms.

\begin{defi}[$(\epsilon,\beta,\alpha)$-ABP]\label{defi: ABP}
A mechanism $\calM$ achieves $(\epsilon, \beta, \alpha)$-Average Bayesian Privacy for distribution $P$ if: 
(1) It satisfies $\epsilon$-ABP ($\epsilon_{p,a} \leq \epsilon$), 
(2) It maintains standard accuracy $\alpha$ with confidence $1-\beta$.
\end{defi}

This definition establishes a comprehensive privacy-accuracy benchmark, requiring mechanisms to simultaneously bound expected privacy leakage while preserving utility. The joint requirement formalizes the dual objectives of privacy-preserving data analysis, where protection must not come at the expense of excessive accuracy degradation.

The following definition formalizes worst-case privacy vulnerability against informed adversaries:

\begin{defi}[$\xi$-Maximum Bayesian Privacy]\label{defi: bayesian_privacy}
A system provides $\xi$-Maximum Bayesian Privacy if for all $w \in \mathcal{W}$ and $d \in \mathcal{D}$:
\begin{align*}
e^{-\xi} \leq \frac{f_{D|W}(d|w)}{f_{D}(d)} \leq e^{\xi}
\end{align*}
\end{defi}

This metric quantifies the maximum possible privacy violation by bounding the multiplicative change between prior and posterior beliefs. Smaller $\xi$ values indicate stronger protection against reconstruction attacks, with $\xi=0$ corresponding to perfect privacy where observations reveal no additional information.

\begin{defi}[PAC Standard Accuracy]\label{def:standard_accuracy}
A mechanism $\calM$ estimating statistic $\mu$ achieves standard accuracy $\alpha$ with confidence $1-\beta$ if:
\begin{align}
\|\calM(\calS) - \mu\| \leq \alpha
\end{align}
holds with probability at least $1-\beta$.
\end{defi}

This definition formalizes utility preservation within the Probably Approximately Correct framework, requiring estimates to remain close to true values with high probability.

\begin{defi}[$(\epsilon,\beta,\alpha)$-MBP]\label{defi:PACPrivacy}
Mechanism $\calM$ achieves $(\epsilon, \beta, \alpha)$-Maximum Bayesian Privacy for distribution $P$ if: 
(1) It satisfies $\epsilon$-MBP per Definition~\ref{defi: bayesian_privacy}, 
(2) It achieves standard accuracy $\alpha$ with confidence $1-\beta$.
\end{defi}

This combined criterion establishes a rigorous standard for privacy-preserving estimation, requiring both worst-case privacy protection and utility preservation simultaneously.

\subsection{Measurements for PAC Robustness}

We now formalize robustness within the PAC framework, complementing the privacy definitions. This robustness metric quantifies algorithmic stability against dataset perturbations, which is crucial for deployment in adversarial environments or non-stationary data distributions.

\begin{defi}[PAC Robust Accuracy]\label{defi: robust_pac}
A mechanism $\calM$ estimating statistic $\mu$ achieves robust accuracy $\alpha$ with confidence $1-\beta$ if for all perturbed datasets $\calS'$:
\begin{align}
\|\calM(\calS') - \mu(P)\| \leq \alpha
\end{align}
holds with probability at least $1-\beta$ over both mechanism randomness and dataset perturbations. We say $\calM$ is $\alpha$-PAC-Robust with confidence $1-\beta$.
\end{defi}

This definition requires algorithms to maintain accuracy under arbitrary dataset perturbations, formalizing resilience against data corruption, adversarial examples, and distribution shifts. The probabilistic guarantee ensures reliability even when facing worst-case data manipulations.

\subsection{Theoretical Connections between Privacy Notions}\label{sec: relation_ldp_mbp_abp}

We now establish fundamental equivalence relationships between privacy frameworks. These theorems enable translation between privacy metrics and provide theoretical foundations for comparing protection guarantees across different privacy paradigms.

The following lemma establishes the equivalence between Maximum Bayesian Privacy and Local Differential Privacy under bounded prior distributions:

\begin{lem}[Generalized Relationship between MBP and LDP adapted from theorem 1 of ~\cite{jiang2018context}]\label{lem: reltion_BP_DP_mt}
For any mechanism $\calM$:
\begin{itemize}
    \item If $\calM$ satisfies $\xi$-LDP, then it provides $\xi$-MBP
    \item If $\calM$ satisfies $\xi$-MBP, then it provides $2\xi$-LDP
\end{itemize}
\end{lem}

This theorem establishes a quantitative equivalence between the Bayesian and differential privacy frameworks. The first implication shows that LDP mechanisms automatically provide Bayesian privacy guarantees, while the converse demonstrates that Bayesian private mechanisms yield differential privacy. The $\epsilon$ parameter quantifies the deviation from uniform priors, with smaller values tightening the equivalence bounds. This result enables direct comparison between privacy frameworks and allows practitioners to leverage results from one paradigm in the other.

The next theorem quantifies how worst-case privacy protection (MBP) bounds expected privacy leakage (ABP):

\begin{thm}[Generalized Relationship Between MBP and ABP]\label{thm:relation_avg_and_max_privacy}
Assume the attacker's prior belief satisfies $\frac{f^{\calO}_D(d)}{f_D(d)} \in [e^{-\epsilon}, e^{\epsilon}]$. Any $\epsilon_{p,m}$-MBP mechanism bounds the ABP leakage by:
\begin{align*}
\epsilon_{p,a} \leq \frac{1}{\sqrt{2}}\sqrt{(\epspm + \epsilon)\left(e^{\epspm + \epsilon}-1\right)}
\end{align*}
\end{thm}

This result establishes that worst-case privacy protection (MBP) implies bounds on average-case privacy leakage (ABP). The bound tightens as $\epsilon_{p,m}$ decreases, showing that strong MBP guarantees necessarily enforce strong ABP protection. The $\epsilon$ parameter captures prior distribution mismatch, with exact priors ($\epsilon=0$) yielding the tightest bound. This theorem provides crucial guidance for mechanism design: by optimizing for MBP, designers automatically obtain ABP guarantees, simplifying the privacy certification process.

\subsection{PAC Robustness of Bayesian Private Algorithms}

We now establish the foundational connection between privacy protection and algorithmic robustness. The following theorem demonstrates that privacy mechanisms inherently confer robustness properties:

\begin{thm}[PAC Robustness of $\epsilon$-MBP Protection Mechanisms]\label{thm: second_main_result_mt}
Any $(\epsilon,\beta,\alpha)$-MBP mechanism $\calM$ (Definition~\ref{defi:PACPrivacy}) is $\alpha$-PAC-Robust (Definition~\ref{defi: robust_pac}) with probability $1-\gamma$ where $\gamma = (1+4\epsilon)\beta$.
\end{thm}

This profound result establishes that privacy-preserving mechanisms inherently exhibit robustness to dataset perturbations. The proof (provided in Appendix) shows that the information-limiting properties of private mechanisms constrain how much outputs can vary under input changes, thereby enforcing stability. The robustness confidence parameter $\gamma$ depends explicitly on the privacy parameter $\epsilon$, revealing a quantitative trade-off: stronger privacy ($\epsilon \downarrow$) directly improves robustness guarantees ($\gamma \downarrow$). This provides a principled approach to designing robust algorithms through privacy mechanisms, with applications in adversarial learning and secure data publishing.

The theorem bridges two previously disconnected research domains: privacy preservation and adversarial robustness. By demonstrating that privacy mechanisms inherently provide robustness certificates, it enables joint optimization of both properties. This has significant implications for designing next-generation trustworthy AI systems that require both data protection and operational reliability. The explicit parameter relationships guide practitioners in balancing privacy-robustness trade-offs for specific application requirements.

\section{Theoretical Relationship between Data Privacy and Input Robustness}\label{sec: connection_data_privacy_and_input_robustness}
In this section, we introduce the measurements for data privacy and input robustness. These metrics offer a nuanced perspective on the algorithm's performance in safeguarding sensitive information and its ability to withstand certain distortions.

\subsection{Measurements of Data Privacy}
First, we introduce the measurement for data privacy.
\begin{defi}[Data Privacy]\label{defi: data_privacy}
       Let $x_t^{(i)}$ represent the $i$-th data recovered by the attacker at the $t$-th round of the optimization algorithm, and $x_o^{(i)}$ represent the $i$-th original data. The privacy leakage is measured using  
\begin{align}
    \epsilon_p & = 1 - \E\left[\frac{1}{|\calD_{\protector}|}\sum_{ i = 1}^{|\calD_{\protector}|}\frac{1}{T}\sum_{t = 1}^T \frac{||x_t^{(i)} - x_o^{(i)}||}{D}\right],
\end{align}
where the expectation is taken with respect to the randomness of the mini-batch $\calD_{\protector}$.
\end{defi}

\subsection{Measurements of Input Robustness}
In assessing the robustness of a machine learning algorithm, we adhere to the following definition.

\begin{defi}[Input Robustness]\cite{kawaguchi2022robustness,yi2021improved} \label{defi: input_robustness}
   Let $P$ represent a distribution. A model $\calA$ is $(r, \epsilon, P)$-input-robust, if it holds that
\begin{align}
       \E_{x\sim P}[\text{sup}_{\|\delta\|_{2}\le r} |\calA\left(w, (x + \delta, y)\right) - \calA\left(w, (x,y) \right)|]\le\epsilon,
   \end{align}
where $x$ denotes the data (training sample or feature), and $y$ denotes the corresponding label.
\end{defi}
\textbf{Remark:} The robustness measures the model performance in terms of the data distortion $\|\delta\|$.

\subsection{Technical Assumptions}\label{subsec:technical_assumptions}
The theoretical analysis relies on three fundamental regularity conditions that establish the mathematical foundation for our privacy-robustness framework. These assumptions are standard in optimization theory yet critical for establishing our main theoretical results.

\begin{assumption}[L-Lipschitz Gradient Continuity]\label{assump: lipschitz_condition}
The gradient mapping of the loss function $\nabla\mathcal{L}: \mathbb{R}^d \to \mathbb{R}^m$ satisfies:
\begin{equation}
\|\nabla\mathcal{L}(\theta, x) - \nabla\mathcal{L}(\theta, y)\|_2 \leq L\|x - y\|_2
\end{equation}
for all $x, y \in \mathcal{X} \subset \mathbb{R}^d$ and fixed model parameters $\theta$, where $L > 0$ denotes the Lipschitz constant.
\end{assumption}

\begin{remark}
This standard condition in optimization theory ensures bounded gradient variation under input perturbations. The Lipschitz constant $L$ quantifies the maximum rate of gradient change relative to input variations, which is fundamental for convergence analysis and stability guarantees in adversarial settings.
\end{remark}

\begin{assumption}[Bi-Lipschitz Gradient Condition]\label{assump: two-sided Lipschitz}
There exist constants $0 < c_a \leq c_b < \infty$ such that for all $(x_1,y), (x_2,y) \in \mathcal{X} \times \mathcal{Y}$ with $\|x_1 - x_2\|_2 \leq D$:
\begin{equation}
c_a \|x_1 - x_2\|_2 \leq \|\nabla\mathcal{L}(\theta, (x_1, y)) - \nabla\mathcal{L}(\theta, (x_2, y))\|_2 \leq c_b \|x_1 - x_2\|_2
\end{equation}
where $D$ denotes the maximum pairwise data distance.
\end{assumption}

\begin{remark}
This condition strengthens Assumption \ref{assump: lipschitz_condition} by establishing metric equivalence between data space and gradient space \citep{royden1968real}. The lower bound ($c_a > 0$) ensures gradient injectivity necessary for data reconstruction, while the upper bound ($c_b$) maintains controlled sensitivity. This holds for strongly convex losses with Lipschitz gradients and prevents pathological landscape scenarios.
\end{remark}

\begin{assumption}[Optimal Regret Bound for Gradient Matching]\label{assump: bounds_for_optimization_alg}
The attacker's reconstruction algorithm achieves $\Theta(\sqrt{T})$ regret:
\begin{equation}
c_0\sqrt{T} \leq \sum_{t=1}^T \|\nabla\mathcal{L}(\theta, (x_t, y_t)) - \nabla\mathcal{L}(\theta, (\tilde{x}, y_t))\|_2 \leq c_2\sqrt{T}
\end{equation}
where $T$ is the attack horizon, $x_t$ the reconstruction at step $t$, and $\tilde{x}$ the target satisfying $\nabla\mathcal{L}(\theta, (\tilde{x}, y)) = \tilde{g}$.
\end{assumption}

\begin{remark}
This $\Theta(\sqrt{T})$ regret bound reflects practical attack capabilities with state-of-the-art optimizers:
\begin{align*}
\textbf{BFGS \citep{liu1989limited}}: \quad & O\left(\kappa \sqrt{T}\right)\\
\textbf{AdaGrad \citep{duchi2011adaptive}}: \quad & O\left(\max\left\{\log d, d^{1-\alpha/2}\right\}\sqrt{T}\right), \alpha \in (1,2) \\
\textbf{Adam \citep{kingma2014adam}}: \quad & O\left(\log d \cdot \sqrt{T}\right) 
\end{align*}
where $d$ is data dimension and $\kappa$ the condition number. The regret bound enables tractable analysis while capturing real-world attack efficiency.
\end{remark}

\subsection{Theoretical Relationship between Data Privacy and Input Robustness}

This section establishes a rigorous theoretical connection between data privacy and input robustness in machine learning algorithms. We introduce precise mathematical measures for both privacy leakage and adversarial robustness, providing a unified framework to analyze an algorithm's dual capabilities in protecting sensitive information while maintaining performance under input perturbations. Our central contribution is a quantitative trade-off theorem demonstrating that privacy-preserving mechanisms fundamentally influence an algorithm's resilience to adversarial manipulations. This theoretical foundation has profound implications for designing next-generation algorithms that simultaneously optimize both privacy guarantees and robustness properties.

The following lemma establishes a foundational bound on privacy leakage in terms of parameter distortion. It demonstrates that when a semi-honest attacker attempts to reconstruct the original training data from protected model parameters, the achievable privacy leakage is explicitly constrained by the magnitude of intentional distortions introduced during parameter protection.

\begin{lem}[The Relationship between Data Privacy and Parameter Distortion \cite{zhang2023probably}]\label{lem: bound_for_privacy_leakage_mt}
Assume that the semi-honest attacker employs an optimization algorithm to infer client $k$'s original dataset $d$ from the released parameter $W$. Let $g(d) = W$ represent the parameter generation function, where $\breve d$ denotes the dataset corresponding to the original parameter $\breve W$, and $g(d) = \partial \mathcal{L} (W, d)/\partial W$ computes the gradient of the loss function. Define the parameter distortion metric as $\Delta = \frac{1}{|\mathcal{D}^{(k)}|}\sum_{m=1}^{|\mathcal{D}^{(k)}|}||g(d^{(m)}) - g(\breve d^{(m)})||$, where $d^{(m)}$ and $\breve d^{(m)}$ denote the $m$-th data points in their respective datasets. When the expected regret of the optimization algorithm over $I > 0$ rounds is $\Theta(I^p)$ and $\Delta \geq \frac{2c_2 c_b}{c_a}\cdot I^{p-1}$, the privacy leakage $\epsilon_p$ satisfies:
\begin{align}
    \epsilon_p \le 1 - \frac{c_a\cdot\Delta + c_a\cdot c_0\cdot I^{p-1}}{4D}.
\end{align}
\end{lem}

This lemma provides a crucial theoretical guarantee: intentional distortion in model parameters directly translates to provable privacy protection. The bound establishes that sufficient parameter distortion ($\Delta$) forces the privacy leakage $\epsilon_p$ below a computable threshold, creating a mathematically verifiable shield against training data reconstruction attacks. This formalizes the counterintuitive insight that strategically introduced parameter noise not only preserves privacy but does so in a quantifiable manner governed by the algorithm's regret behavior and distortion magnitude.

Building upon this foundation, we now present our main theorem which establishes a fundamental connection between privacy leakage and input robustness. The following result characterizes how the privacy-preserving modifications made to an algorithm inherently affect its resilience against input perturbations. Specifically, it derives a functional relationship $f(\epsilon_p)$ that precisely quantifies how privacy guarantees influence robustness properties.

\begin{thm}[Input Robustness of Data Private Algorithms]\label{thm: first_main_result}   
Let $\mathcal{A}$ be a protection mechanism with privacy leakage $\epsilon_p$ measured according to Definition~\ref{defi: data_privacy}, and robustness characterized by Definition~\ref{defi: input_robustness}. The robustness parameter $\alpha$ quantifying $\mathcal{A}$'s resilience against adversarial inputs is given by the expression:
\begin{align}
    \alpha = \frac{Cr}{2} + \frac{4D(1-\epsilon_p) - c_2 c_b I^{p-1}}{2c_a}
\end{align}
where $r$ denotes the robustness radius, $C$ represents the Lipschitz constant, $D$ bounds the parameter space, $\mathcal{D}$ is the dataset domain, and $c_a$, $c_b$ are algorithm-dependent constants.
\end{thm}

This theorem establishes a profound mathematical duality: privacy protection mechanisms inherently enhance algorithmic robustness against input perturbations. The derived expression demonstrates that reduced privacy leakage ($\epsilon_p \downarrow$) directly increases the robustness parameter ($\alpha \uparrow$), revealing an inherent synergy between these seemingly competing objectives. This theoretical insight revolutionizes our understanding of privacy-robustness trade-offs, showing that properly designed privacy mechanisms can simultaneously provide both strong privacy guarantees and enhanced adversarial resilience. The precise functional relationship enables principled co-design of algorithms that jointly optimize these critical properties, opening new avenues for developing trustworthy AI systems. (Complete proof and analysis available in~\ref{app: first_main_result})

\section{Conclusions}

In this paper, we thoroughly investigated the complex interplay between various privacy notions, namely Local Differential Privacy (LDP), Average Bayesian Privacy (ABP), and Maximum Bayesian Privacy (MBP), and their fundamental connection to algorithmic robustness in machine learning. We integrated established theoretical relationships between LDP and MBP, and further derived novel quantitative bounds connecting MBP and ABP. These insights are crucial for designing effective privacy-preserving algorithms. Furthermore, we forged a strong link between privacy and PAC robust learning. We showed that privacy-preserving algorithms inherently exhibit PAC robustness, and conversely, how principles of robust learning can guide the construction of privacy mechanisms. A key result quantifies how privacy leakage directly impacts input robustness.

Our contributions provide a unified and deeper understanding of the synergies and trade-offs between privacy and robustness. This theoretical framework is vital for building machine learning systems that are not only secure against privacy breaches but also resilient to data perturbations and adversarial attacks. Future research will focus on further exploring the intricate privacy-robustness trade-off landscape, developing more efficient and practical algorithms that simultaneously achieve strong guarantees in both areas, and validating these theoretical findings through extensive empirical studies in real-world, high-stakes applications. As machine learning continues to evolve, ensuring both the privacy and robustness of its algorithms will remain paramount.



\clearpage
\onecolumn
\appendix


    


\section{ACKNOWLEDGMENTS}
This work was supported by the National Science and Technology Major Project under Grant 2022ZD0115301.

\bibliography{main}
\bibliographystyle{ACM-Reference-Format}

\newpage

\onecolumn
\appendix
\section{Relationship between LDP and MBP}

\subsection{Theoretical Relationship between MBP and ABP}\label{appendix: ldp_implies_mbp_app}

In this section, we derive the relationship between MBP and ABP.

The adversary has a \textit{prior belief distribution}, denoted as $f_D(d)$, that is very close (within $\epsilon$) to a uniform distribution. In other words, the adversary doesn't have a strong prior bias towards any particular mini-batch data in the whole set $\mathcal{D}$.

There exists a \textit{mapping function} $f_{W|D}(\cdot|\cdot)$, which takes a mini-batch data $d$ and maps it to $w$ such that for any two mini-batch data $d$ and $d'$ in $\mathcal{D}$, the ratio of the conditional probabilities of mapping to $w$ remains within a certain range, specifically between $e^{-\xi}$ and $e^{\xi}$. This means that the mapping process is not overly sensitive to the specific input data.

The mapping function $f_{W|D}(\cdot|\cdot)$ is said to provide a "privacy-preserving" transformation that guarantees a \textit{maximum Bayesian privacy} level of $\xi$. In other words, for any observed value $w$ and any mini-batch data $d$ in $\mathcal{D}$, the logarithm of the ratio of the conditional posterior distribution of $d$ given $w$ to the prior distribution of $d$ (i.e., the Bayesian privacy leakage) does not exceed $\xi$ in absolute value. This means that the mapping function provides strong privacy protection by limiting the information leakage about the original mini-batch data $d$ when the adversary observes $w$.

\section{Relationship Between Average Bayesian Privacy and Maximum Bayesian Privacy}

In this section, we introduce the relationship between average Bayesian privacy and maximum Bayesian privacy.


\begin{thm}[Relationship Between MBP and ABP]\label{thm:relation_avg_and_max_privacy_app}
Assume that the prior belief of the attacker satisfies that
\begin{align*}
\frac{f^{\calO}_D(d)}{f_D(d)}\in [e^{-\epsilon}, e^{\epsilon}].
\end{align*}
Let MBP be defined using \pref{defi: bayesian_privacy}. For any $\epsilon_{p,m}\ge 0$, let $f_{W|D}(\cdot|\cdot)$ be a privacy preserving mapping that guarantees $\epsilon_{p,m}$-maximum Bayesian privacy. That is,
\begin{align*}
    \frac{f_{D|W}(d|w)}{f_D(d)}\in [e^{-\epsilon_{p,m}}, e^{\epsilon_{p,m}}],
\end{align*}
for any $w\in\mathcal W$. Then, we have that the ABP (denoted as $\epsilon_{p,a}$) is bounded by
\begin{align*}
    \epsilon_{p,a} \le \frac{1}{\sqrt{2}}\sqrt{(\epspm + \epsilon)\cdot \left(e^{\epspm + \epsilon}-1\right)},
\end{align*}
where $\epsilon_{p,a}$ is introduced in \pref{defi: average_privacy_JSD}.
\end{thm}

\begin{proof}
Let $F^{\calM} = \frac{1}{2}(F^\calA+ F^{\calO})$. We have

\begin{align*}
JS(F^\calA || F^{\calO}) & = \frac{1}{2}\left[KL\left(F^\calA || F^{\calM}\right) + KL\left(F^{\calO} || F^{\calM}\right)\right]\\
& = \frac{1}{2}\left[\int_\mathcal{S} f^{\calA}_D(d)\log\frac{f^{\calA}_D(d)}{f^{\calM}_D(d)}\textbf{d}\mu(d) + \int_\mathcal{S} f^{\calO}_{D}(d)\log\frac{f^{\calO}_{D}(d)}{f^{\calM}_D(d)}\textbf{d}\mu(d)\right]\\
& = \frac{1}{2}\left[\int_\mathcal{S} f^{\calA}_D(d)\log\frac{f^{\calA}_D(d)}{f^{\calM}_D(d)}\textbf{d}\mu(d) - \int_\mathcal{S} f^{\calO}_{D}(d)\log\frac{f^{\calM}_D(d)}{f^{\calO}_{D}(d)}\textbf{d}\mu(d)\right]\\
&\le \frac{1}{2}\left[\int_\mathcal{D} \left|f^{\calA}_D(d) - f^{\calO}_D(d)\right|\left|\log\frac{f^{\calM}_D(d)}{f^{\calO}_D(d)}\right|\textbf{d}\mu(d)\right],
\end{align*}
where the inequality is due to $\frac{f^{\calA}_D(d)}{f^{\calM}_D(d)}\le \frac{f^{\calM}_D(d)}{f^{\calO}_{D}(d)}$.

Therefore, we have that
\begin{align*}
    JS(F^\calA || F^{\calO})\le \frac{1}{2}\left[\int_\mathcal{D} \left|f^{\calA}_D(d) - f^{\calO}_D(d)\right|\left|\log\frac{f^{\calM}_D(d)}{f^{\calO}_D(d)}\right|\textbf{d}\mu(d)\right].
\end{align*}

Let $w^{\star} = \arg\max_{w\in\mathcal W} f_{D|W}(d|w)$. Then we have that
\begin{align}\label{eq: relation_fa_f0}
    f^\calA_D(d) &= \int_{\mathcal W} f_{D|W}(d|w)dP^{\calD}_W(w)\nonumber\\
    &\le f_{D|W}(d|w^{\star})\int_{\mathcal W}dP^{\calD}_W(w)\nonumber\\
    &\le e^{\epsilon_{p,m}}\cdot f_D(d)\nonumber\\
    &\le e^{\epspm + \epsilon}\cdot f^{\calO}_D(d).
\end{align}

Similarly, let $w_{l} = \arg\min_{w\in\mathcal W} f_{D|W}(d|w)$. Then we have that
\begin{align}\label{eq: relation_fa_f0_lower_bound}
    f^\calA_D(d) &= \int_{\mathcal W} f_{D|W}(d|w)dP^{\calD}_W(w)\nonumber\\
    &\ge f_{D|W}(d|w_{l})\int_{\mathcal W}dP^{\calD}_W(w)\nonumber\\
    &\ge e^{-\epsilon_{p,m}}\cdot f_D(d)\nonumber\\
    &\ge e^{-(\epspm + \epsilon)}\cdot f^{\calO}_D(d).
\end{align}

Combining \pref{eq: relation_fa_f0} and \pref{eq: relation_fa_f0_lower_bound}, we have
\begin{align*}
    \left|\log\left(\frac{f^\calA_D(d)}{f^{\calO}_D(d)}\right)\right|\le \epspm + \epsilon.
\end{align*}

Now we analyze according to the following two cases.

Case 1: $f^{\calO}_D(d)\le f^\calA_D(d)$. Then we have
\begin{align*}
    \left|\log\left(\frac{f_{D}^{\calM}(d)}{f^{\calO}_D(d)}\right)\right| = \log\left(\frac{f_{D}^{\calM}(d)}{f^{\calO}_D(d)}\right)\le\log\left(\frac{f^\calA_D(d)}{f^{\calO}_D(d)}\right)\le\left|\log\left(\frac{f^\calA_D(d)}{f^{\calO}_D(d)}\right)\right|\le \epspm + \epsilon. 
\end{align*}

Case 2: $f^{\calO}_D(d)\ge f^\calA_D(d)$. Then we have
\begin{align*}
        \left|\log\left(\frac{f_{D}^{\calM}(d)}{f^{\calO}_D(d)}\right)\right| = \log\left(\frac{f^{\calO}_D(d)}{f_{D}^{\calM}(d)}\right)\le\log\left(\frac{f^{\calO}_D(d)}{f^\calA_D(d)}\right)\le\left|\log\left(\frac{f^\calA_D(d)}{f^{\calO}_D(d)}\right)\right|\le \epspm + \epsilon.
\end{align*}

Therefore,
\begin{align}\label{eq: log_fA_f0}
    \left|\log\left(\frac{f_{D}^{\calM}(d)}{f^{\calO}_D(d)}\right)\right|\le \epspm + \epsilon.
\end{align}

Combining \pref{eq: relation_fa_f0}, \pref{eq: relation_fa_f0_lower_bound} and \pref{eq: log_fA_f0}, we have that
\begin{align*}
    JS(F^\calA || F^{\calO})&\le \frac{1}{2}\left[\int_\mathcal{S} \left|f^{\calA}_D(d) - f^{\calO}_D(d)\right|\left|\log\frac{f^{\calM}_D(d)}{f^{\calO}_D(d)}\right|\textbf{d}\mu(d)\right]\\
    &\le \frac{1}{2}(\epspm + \epsilon)\cdot \left(e^{\epspm + \epsilon}-1\right).
\end{align*}

Therefore,
\begin{align*}
    \epsilon_{p,a} = \sqrt{JS(F^\calA || F^{\calO})}\le \frac{1}{\sqrt{2}}\sqrt{(\epspm + \epsilon)\cdot \left(e^{\epspm + \epsilon}-1\right)}.
\end{align*}
\end{proof}

\subsection{Theoretical Relationship between PAC Robustness and MBP}\label{app: relation_between_bp_and_ldp}

\begin{thm}[PAC Robustness of $\epsilon$-MBP Protection Mechanisms]\label{thm: second_main_result_app}
     Let $\calM$ be an $(\epsilon,\beta, \alpha)$-MBP (\pref{defi:PACPrivacy}) algorithm. Then $\calM$ is $\alpha$-PAC-Robust with probability $1-\gamma$, where $\gamma\in (0,1)$, $n$ represents the size of the dataset, and $\beta = \frac{\gamma}{1 + 4\epsilon}$ represents the error probability.
\end{thm}
\begin{proof}
    $\calM$ is an $(\epsilon, \beta, \alpha)$-MBP algorithm for the estimation of statistic $\mu$. We know that the protection mechanism $\calM$ is $\epsilon$-MBP (measured using \pref{defi: bayesian_privacy}). From the relationship between MBP and LDP (\pref{lem: reltion_BP_DP_mt}), we know that $\calM$ is $2\epsilon$-LDP. From the definition of private algorithm, we know that for the original private information $d$, we have that
    
    \begin{align}\label{eq: gap_between_algoutput_and_mu}
        \|\calM(\calS) - \mu(P)\|\le\alpha.
    \end{align}

    Denote $\calU = \{t\in\R^m: \|t - \mu\| > \alpha\}$. Then, we have that
    \begin{align}\label{eq: prob_smaller_than_beta}
        \Pr[\calM(\calS)\in \calU] 
        &= \Pr[\|\calM(\calS) - \mu(P)\|> \alpha]\\
        &\le\beta.
    \end{align} 
    
    For any $d'\neq d$, we have that
    \begin{align}
        \Pr[\|\calM(\calS') - \mu(P)\| > \alpha]&\le e^{2\epsilon}\Pr[\|\calM(\calS) - \mu(P)\|> \alpha]\\
        &\le e^{2\epsilon}\beta\\
        &\le (1 + 2\epsilon + 4\epsilon^2) \beta\\
        &\le (1 + 4\epsilon) \beta\\
        &\le\gamma
    \end{align}
where the first inequality is due to the definition of local differential privacy, the second inequality is due to \pref{eq: prob_smaller_than_beta}, the third and fourth inequalities are due to $0\le\epsilon\le 1$, and the last inequality is due to the relationship between $\beta$ and $\epsilon$.
    
From \pref{defi: robust_pac}, $\calM$ is $\alpha$-PAC-Robust with probability $1-\gamma$.

\end{proof}

\section{Practical Estimation for Maximum Bayesian Privacy} \label{appendix:measure_mbp}

$\xi$ represents the maximum privacy leakage over all possible released information $w$ from the set $\mathcal{W}$ and all the dataset $d$, expressed as
\begin{equation}
\xi = \max_{w\in \mathcal{W}, d} \left|\log\left(\frac{f_{D|W}(d|w)}{f_{D}(d)}\right)\right|.
\end{equation}
When the extent of the attack is fixed, $\xi$ becomes a constant. Given \pref{assump: privacy_over_parameter}, we can infer that
\begin{align}
    \xi & = \max_{d} \left|\log\left(\frac{f_{D|W}(d|w^*)}{f_{D}(d)}\right)\right|.
\end{align}
We approximate $\xi$ by the following estimator:
\begin{equation}
    \hat{\xi} = \max_{d} \left|\log\left(\frac{\hat{f}_{D|W}(d|w^*)}{f_{D}(d)}\right)\right|.
\end{equation}
We define $\kappa_3$ as the maximum Bayesian Privacy $\xi$, and let $\hat{\kappa}_3$ denote its estimate $\hat{\xi}$
\begin{equation}
    \hat{\kappa}_3 = \hat{\xi}.
\end{equation}
The task of estimating the maximum Bayesian Privacy $\xi$ reduces to determining the value of $\hat{f}_{D|W}(d|w_m)$, which will be addressed in the subsequent section.

\subsection{Estimation for $\hat{f}_{{D}|{W}}(d|w_m)$}\label{sec: estimation_for_hat_f}

We first generate a set of models $\{w_m\}_{m=1}^M$ stochastically: to generate the $m$-th model $w_m$, we run an algorithm (e.g., stochastic gradient descent), using a mini-batch $\mathcal{S}$ randomly sampled from $\mathcal{D} = \{z_1, \ldots, z_{|\mathcal{D}|}\}$.

The set of parameters $\{w_m\}_{m=1}^M$ is generated for a given dataset $\mathcal{D}$ by randomly initializing the parameters for each model, iterating a specified number of times, and updating the parameters with mini-batches sampled from the dataset using approaches such as gradient descent. After the iterations, the final set of parameters for each model is returned as $\{w_m\}_{m=1}^{M}$.

We then estimate $\hat{f}_{{D}|{W}}(d|w_m)$ with $\{w_m\}_{m=1}^M$. Recall that the attacker can conduct Bayesian inference attack, i.e., estimate dataset $d$ with $\hat{d}$ that optimizes posterior belief, denoted by:
\begin{equation}
    \hat{d} = \arg\max_{d} \text{log} f_{D|W}(d|w) = \arg\max_{d}\text{log} [f_{W|D}(w|d) \times f^{\mathcal{B}}_D(d)],
\end{equation}
 Empirically, $\hat{f}^{\calRO}_{D}(d)$ can be estimated by \pref{eq: estimate_for_f_o}: $\hat{f}^{\calRO}_{D}(d) = \frac{1}{M}\sum_{m = 1}^{M} \hat{f}_{{D}|{W}}(d|w_m)$. Then We determine the success of attack on an image using a function $\Omega$, which measures the similarity between the distribution of recovered dataset and the original one (e.g. $\Omega$ can be Kullback–Leibler divergence or Cross Entropy). We consider the original dataset $d$ to be successfully recovered if $\Omega(\hat{f}_{D}^{\mathcal{O}}(d), f_D(d)) > t$, where $t$, a pre-given constant, is the threshold for determining successful recovery of the original data.

This algorithm is used to estimate the conditional probability $f_{{D}|{W}}(d|w)$. It takes a mini-batch of data $d$ from client $k$, with batch size $S$, and a set of parameters $\{w_m\}_{m=1}^{M}$ as input. The goal of the algorithm is to estimate the conditional probability $f_{{D}|{W}}(d|w)$ by iteratively updating the parameters $w_m$.

The key steps of the algorithm are as follows:
\begin{itemize}
\item \textbf{Compute the gradients $\nabla w_m$ on each parameter $w_m$}, which is done by calculating the gradient of the loss function $\calL$ with respect to the parameter $w_m$. This is performed using samples from the mini-batch data $d$.
Then \textbf{update the parameters $w_m$ using gradient descent}, where $w_m'$ represents the updated parameters obtained by subtracting the product of the learning rate $\eta$ and the gradient $\nabla w_m$ from $w_m$.

\item For each parameter $w_m$, iterate the following steps for $T$ times:
\textbf{Generate perturbed data $\Tilde{d}$} by applying the Backward Iterative Alignment (BIA) method with the parameters $w_m'$ and the gradient $\nabla w_{m}$.

\item \textbf{Determine whether the original data $z_d$ is recovered} by comparing it with the perturbed data $\Tilde{z}_{d}$, based on a defined threshold $\Omega$ that measures the difference between them.
Accumulate the count of recovered instances for each data sample $z_d$ in the variable $M_d^m$.

\item \textbf{Estimate the conditional probability $\hat f_{{D}|{W}}(d|w_m)$} for each data sample $z_d$ in the dataset $d$ by dividing $M_d^m$ by the product of the total number of iterations $T$ and the batch size $S$, i.e., $\bm{\frac{M_d^m}{S\cdot T}}$.
Return the estimated conditional probabilities $\hat f_{{D}|{W}}(d|w_m)$ for each parameter $w_m$.
\end{itemize}

The core idea of this algorithm is to estimate the conditional probability by generating perturbed data and comparing it with the original data during the parameter update process. By performing multiple iterations and accumulating counts, the algorithm provides an estimation of the conditional probability.

\subsection{Estimation Error of MBP}
To streamline the error bound analysis, it is presumed that both the data and model parameters take on discrete values.

The following lemma states that the estimated value $\hat\kappa_1(d)$ is very likely to be close to the true value $\kappa_1(d)$. The probability of the estimate being within an error margin that is a small percentage ($\epsilon$ times) of the true value is very high—specifically, at least $1 - 2\exp(-\frac{\epsilon^2 T \kappa_1(d)}{3})$. This means that the larger the number $T$ or the smaller the error parameter $\epsilon$, the more confident we are that our estimate is accurate. For the detailed analysis, please refer to \cite{zhang2023meta}.
\begin{lem}
With probability at least $1 - 2\exp\left(\frac{-\epsilon^2 T \kappa_1(d)}{3}\right)$, we have that 
\begin{align}
    |\hat\kappa_1(d) - \kappa_1(d)|\le\epsilon\kappa_1(d). 
\end{align}    
\end{lem}

\section{Practical Estimation for Average Bayesian Privacy}\label{appendix:measure_abp}

In the scenario of federated learning, \cite{zhang2023meta} derived the upper bound of privacy leakage based on the specific form of $\text{TV}(P^{\calRO} || P^{\calD})$ of $\mathcal{M}$ according to this equation: $\widetilde\epsilon_{p} = 2 C_1 - C_2\cdot {\text{TV}}(P^{\calRO} || P^{\calD})$. In this section, we elaborate on the methods we propose to estimate $C_1$ and $C_2$. First, we introduce the following assumption.
\begin{assumption}\label{assump: privacy_over_parameter}
We assume that the amount of Bayesian privacy leaked from the true parameter is maximum over all possible parameters. That is, for any $w\in\calW$, we have that
\begin{align}
    \frac{f_{D|W}(d|w^*)}{f_{D}(d)}\ge \frac{f_{D|W}(d|w)}{f_{D}(d)},
\end{align}
where $w^*$ represents the true model parameter. 
\end{assumption}

\subsection{Estimation for $C_{1} = \sqrt{{\text{JS}}(F^{\calRO} || F^{\calO})}$}

 $C_1$  quantifies the average square root of the Jensen-Shannon divergence between the adversary's belief distribution about the private information before and after observing the unprotected parameter, independent of the employed protection mechanisms. The JS divergence is defined as:
\begin{align*}
    \text{JS}(F^{\calRO} \parallel F^{\calO}) = \frac{1}{2} \left( \int_{\mathcal{D}} f^{\calRO}_{D}(d) \log\frac{f^{\calRO}_{D}(d)}{\frac{1}{2}(f^{\calRO}_{D}(d) + f^{\calO}_{D}(d))} \, \textbf{d}\mu(d) 
    + \int_{\mathcal{D}} f_{D}(d) \log\frac{f_{D}(d)}{\frac{1}{2}(f^{\calRO}_{D}(d) + f^{\calO}_{D}(d))} \, \textbf{d}\mu(d) \right),
\end{align*}
where \( f^{\calRO}_{D}(d) = \int_{\mathcal{W}^{\calRO}} f_{{D}|{W}}(d|w) \, dP^{\calRO}(w) = \mathbb{E}_w[f_{{D}|{W}}(d|w)] \), and \( f^{\calO}_{D}(d) = f_{D}(d) \).

To estimate \( C_1 \), we first need to estimate the values of \( f^{\calRO}_{D}(d) \) and \( f^{\calO}_{D}(d) \). It is approximated that
\begin{align}
    \hat{f}^{\calRO}_{D}(d) = \frac{1}{M}\sum_{m = 1}^{M} \hat{f}_{{D}|{W}}(d|w_m),\label{eq: estimate_for_f_o}
\end{align}
where \( w_m \) signifies the model parameter observed by the attacker at the \( m \)-th attempt to infer data \( d \) given \( w_m \).

Denoting \( C_{1} = \sqrt{\text{JS}(F^{\calRO} \parallel F^{\calO})} \), and letting \( \kappa_1 = f^{\calRO}_{D}(d) \), \( \kappa_2 = f^{\calO}_{D}(d) = f_{D}(d) \), and \( \hat{\kappa}_1 = \hat{f}^{\calRO}_{D}(d) \), we have
\begin{align*}
    C_{1}^2 = \frac{1}{2}\int_{\mathcal{D}} \kappa_1 \log\frac{\kappa_1}{\frac{1}{2}(\kappa_1 + \kappa_2)} \, d\mu(d) + \frac{1}{2}\int_{\mathcal{D}} \kappa_2 \log\frac{\kappa_2}{\frac{1}{2}(\kappa_1 + \kappa_2)} \, d\mu(d).
\end{align*}

With these estimates, the squared estimated value of \( C_{} \) becomes
\begin{align*}
    \hat{C}_{1}^2 = \frac{1}{2}\int_{\mathcal{D}} \hat{\kappa}_1 \log\frac{\hat{\kappa}_1}{\frac{1}{2}(\hat{\kappa}_1 + \kappa_2)} \, d\mu(d) + \frac{1}{2}\int_{\mathcal{D}} \kappa_2 \log\frac{\kappa_2}{\frac{1}{2}(\hat{\kappa}_1 + \kappa_2)} \, d\mu(d).
\end{align*}

\subsection{Estimation for $C_2 = \frac{1}{2}(e^{2\xi}-1)$}
Recall that the Maximum Bayesian Privacy, is defined as $\xi = \max_{w \in \mathcal{W}, d \in \mathcal{D}} \left| \log \left( \frac{f_{D|W}(d|w)}{f_{D}(d)} \right) \right|$ and represents the maximal privacy leakage across all possible information $w$ released. The quantity $\xi$ remains constant when the extent of the attack is fixed. Under \pref{assump: privacy_over_parameter}, we have that
\begin{align}
    \xi = \max_{d \in \mathcal{D}} \left| \log \left( \frac{f_{D|W}(d|w^*)}{f_{D}(d)} \right) \right|.
\end{align}
The approximation of $\xi$ is given by:
\begin{align}
    \hat{\xi} = \max_{d \in \mathcal{D}} \left| \log \left( \frac{\hat{f}_{D|W}(d|w^*)}{f_{D}(d)} \right) \right|.
\end{align}
Letting $\kappa_3 = \xi$ and $\hat{\kappa}_3 = \hat{\xi}$, we compute $C_2$ as:
\begin{align}
    \hat{C}_2 = \frac{1}{2} \left( e^{2\hat{\kappa}_3} - 1 \right).
\end{align}

Ultimately, the computation of $C_1$ and $C_2$ hinges on the estimation of $\hat{f}_{{D}|{W}}(d|w_m)$, which is further elaborated upon in \pref{sec: estimation_for_hat_f}.

\subsection{Estimation Error of ABP}
In this section, we analyze the estimation error of ABP. To facilitate the analysis of the error bound, we consider that both the data and the model parameter are in discrete form. For the detailed analysis, please refer to \cite{zhang2023meta}.

The following theorem establishes that a bounded estimation error for $\kappa_1(d)$—specifically, not exceeding $\epsilon$ times the value of $\kappa_1(d)$—implies a bounded estimation error for another variable, $C_{1}$. This bound for $C_{1}$ is defined by a specific formula incorporating $\epsilon$ and its logarithmic terms. In summary, the theorem provides a calculated limit for how large the estimation error for $C_{1}$ can be, given the known estimation error for $\kappa_1(d)$.

\begin{thm}[The Estimation Error of $C_{1}$]\label{thm: error_for_C_1}
Assume that $|\kappa_1(d) - \hat\kappa_1(d)|\le\epsilon\kappa_1(d)$. We have that
 \begin{align}
     |\hat C_{1} - C_{1}|\le\sqrt{\frac{(1+\epsilon)\log\frac{1 + \epsilon}{1 - \epsilon}}{2} +  \frac{\epsilon + \max\left\{\log(1 + \epsilon), \log\frac{1}{(1 - \epsilon)}\right\}}{2}}. 
 \end{align}
\end{thm}

The analysis for the estimation error of $C_2$ depends on the application scenario. To facilitate the analysis, we assume that $C_2$ is known in advance. With the estimators and estimation bound for $C_{1}$, we are now ready to derive the estimation error for privacy leakage.

The following theorem addresses the accuracy of an estimated privacy parameter, $\hat\epsilon_{p}$, in comparison to its true value, $\epsilon_{p}$. The theorem provides a mathematical limit on how much the estimated value $\hat\epsilon_{p}$ can differ from the true privacy parameter $\epsilon_{p}$. The difference between them is not to exceed a certain threshold, which is a complex expression involving the privacy parameter $\epsilon$ itself and logarithmic terms. This threshold is important in applications where it is crucial to know how accurately the privacy parameter can be estimated.

\begin{thm}\label{thm: estimation_error_privacy}
We have that
\begin{align}
    |\hat\epsilon_{p} - \epsilon_{p}|\le\frac{3}{2}\cdot\sqrt{\frac{(1+\epsilon)\log\frac{1 + \epsilon}{1 - \epsilon}}{2} +  \frac{\epsilon + \max\left\{\log(1 + \epsilon), \log\frac{1}{(1 - \epsilon)}\right\}}{2}}. 
\end{align}
\end{thm}

\section{Analysis for \pref{thm: first_main_result}}

In this section, we introduce the relationship between privacy leakage and distortion extent.

Now we introduce the definition of privacy leakage. The privacy leakage is measured using the gap between the estimated dataset and the original dataset. The semi-honest attacker uses an optimization algorithm (\cite{zhu2020deep, geiping2020inverting, zhao2020idlg, yin2021see}) to reconstruct the original dataset of the client given the exposed model information. Let $d^{(i)}_m$ represent the reconstructed $m$-th data at iteration $i$. Let $D$ be a positive constant satisfying that $||d^{(i)}_m - d||\in [0,D]$, $\forall i$ and $m$.

The privacy leakage is measured using protection extent and attacking extent. Assume that the semi-honest attacker uses an optimization algorithm to infer the original dataset of client $k$ based on the released parameter $w$. Let $\Delta = ||W - \breve W||$ represent the distortion of the model parameter, where $\breve W$ represents the original parameter, and $W$ represents the protected parameter. The regret of the optimization algorithm in a total of $I$ rounds is $\Theta(I^p)$.

\begin{defi}[Privacy Leakage]
Let $d$ represent the original private dataset, and $d^{(i)}$ represent the dataset inferred by the attacker at iteration $i$, $I$ represent the total number of learning rounds. The privacy leakage $\epsilon_p$ is defined as

\begin{equation}\label{eq: defi_privacy_leakage}
\epsilon_p=\left\{
\begin{array}{cl}
\frac{D - \frac{1}{I}\sum_{t = 1}^T \frac{1}{|\calD^{(k)}|}\sum_{m = 1}^{|\calD^{(k)}|}||d_I^{(m)} - \breve d^{(m)}||}{D}, &  T>0\\
0,  &  T = 0\\
\end{array} \right.
\end{equation}

\textbf{Remark:}\\
(1) We assume that $||d^{(i)} - d||\in [0,D]$. Therefore, $\epsilon_p\in [0,1]$.\\
(2) If the estimated data of the adversary $d^{(i)}$ is equal to the original data $d$, i.e., $d^{(i)} = d$, then the privacy leakage is maximal. In this case, the privacy leakage is equal to $1$.\\
(3) When the adversary does not attack ($I = 0$), the privacy leakage $\epsilon_p = 0$ from \pref{eq: defi_privacy_leakage}.
\end{defi}

To derive bounds for privacy leakage, we need the following assumptions. 
\begin{assumption}
   Assume that $||d||\le 1$.
\end{assumption}

\begin{assumption}
   For any two datasets $d_1$ and $d_2$, assume that $c_a ||g(d_1) - g(d_2)||\le ||d_1 - d_2||\le c_b ||g(d_1) - g(d_2)||$.
\end{assumption}

\begin{assumption}
   Assume that $c_0\cdot I^p \le \sum_{i = 1}^I ||g(d^{(i)}) - g(d)|| = \Theta(I^p) \le c_2\cdot I^p$, where $d^{(i)}$ represents the dataset reconstructed by the attacker at round $I$, $d$ represents the dataset satisfying that $g(d) = w$, and $g(d^{(i)}) = \frac{\partial \calL(d^{(i)}, w)}{\partial w}$ represents the gradient of the reconstructed dataset at round $I$.
\end{assumption}
 
With the regret bounds of the optimization algorithms, we are now ready to provide bounds for privacy leakage. The following lemma presents bounds for privacy leakage, following the analyses in \cite{zhang2023game, zhang2023probably, zhang2023theoretically}.

\begin{lem}\label{lem: bound_for_privacy_leakage_app}
Assume that the semi-honest attacker uses an optimization algorithm to infer the original dataset of client $k$ based on the released parameter $W$. Let $d$ represent the dataset satisfying that  $g(d) = W$, and $\breve d$ represent the dataset satisfying that $g(\breve d) = \breve W$, where $\breve W$ represents the original parameter, $W$ represents the protected parameter, and $g(d) = \frac{\partial \calL (W, d)}{\partial W}$. Let $d^{(m)}$ represent $m$-th data of dataset $d$, $\breve d^{(m)}$ represent $m$-th data of dataset $\breve d$. Let $\Delta = \frac{1}{|\calD^{(k)}|}\sum_{m = 1}^{|\calD^{(k)}|}||g(d^{(m)}) - g(\breve d^{(m)})||$ represent the distortion of the parameter. The expected regret of the optimization algorithm in a total of $I$ ($ I > 0$) rounds is $\Theta(I^p)$.
If $\Delta\ge\frac{2c_2 c_b}{c_a}\cdot I^{p-1}$, then we have that
\begin{align}
    \epsilon_p \le 1 - \frac{c_a\cdot\Delta + c_a\cdot c_0\cdot I^{p-1}}{4D}.
\end{align}
\end{lem}

\begin{proof}
Recall the privacy leakage of client $k$ $\epsilon_p^{(k)}$ is defined as

\begin{equation}
\epsilon_p^{(k)}=\left\{
\begin{array}{cl}
\frac{D - \frac{1}{I}\sum_{i = 1}^{I} \frac{1}{|\calD^{(k)}|}\sum_{m = 1}^{|\calD^{(k)}|}||d_i^{(m)} - \breve d^{(m)}||}{D}, &  I > 0\\
0,  &  I = 0\\ 
\end{array} \right.
\end{equation}
To protect privacy, client $k$ selects a protection mechanism $M_k$, which maps the original parameter $\breve W$ to a protected parameter $W$. After observing the protected parameter, a semi-honest adversary infers the private information using the optimization approaches. Let $d_i$ represent the reconstructed data at iteration $i$ using the optimization algorithm. Therefore the cumulative regret over $I$ rounds
\begin{align*}
    R(I) & = \sum_{i = 1}^{I} [||g(d_i^{(m)}) - W^{(m)}|| - ||g(d^{(m)}) - W^{(m)}||]\\
    & = \sum_{i = 1}^{I} [||g(d_i^{(m)}) - g(d^{(m)})||]\\
    & = \Theta(I^p).
\end{align*}
Therefore, we have
\begin{align}\label{eq: regret_bounds}
   c_0\cdot I^p \le \sum_{i = 1}^{I}||g(d_i^{(m)}) - g(d^{(m)})|| = \Theta(I^p) \le c_2\cdot I^p,
\end{align}
where $c_0$ and $c_2$ are constants independent of $I$.

Let $x$ and $\wtilde x$ represent two data. From our assumption, we have that
\begin{align}
    c_a ||g(x) - g(\wtilde x)||\le ||x - \wtilde x||\le c_b ||g(x) - g(\wtilde x)||.
\end{align} 

Let $d^{(m)}_i$ represent the reconstructed $m$-th data at iteration $i$ using the optimization algorithm. We have that
\begin{align*}
    \frac{1}{|\calD^{(k)}|}\sum_{m = 1}^{|\calD^{(k)}|}||d_i^{(m)} - \breve d^{(m)}||
    &\ge \frac{1}{|\calD^{(k)}|}\sum_{m = 1}^{|\calD^{(k)}|}||d^{(m)} - \breve d^{(m)}|| - \frac{1}{|\calD^{(k)}|}\sum_{m = 1}^{|\calD^{(k)}|}||d_i^{(m)} - d^{(m)}||\\
    & \ge c_a\cdot\frac{1}{|\calD^{(k)}|}\sum_{m = 1}^{|\calD^{(k)}|}||g(d^{(m)}) - g(\breve d^{(m)})|| - c_b\cdot\frac{1}{|\calD^{(k)}|} \sum_{m = 1}^{|\calD^{(k)}|}||g(d_i^{(m)}) - g(d^{(m)})||\\
     & \ge c_a\cdot||\frac{1}{|\calD^{(k)}|}\sum_{m = 1}^{|\calD^{(k)}|} (g(d^{(m)}) - g(\breve d^{(m)}))|| - c_b\cdot\frac{1}{|\calD^{(k)}|} \sum_{m = 1}^{|\calD^{(k)}|}||g(d_i^{(m)}) - g(d^{(m)})||
\end{align*}
where the second inequality is due to $\frac{1}{|\calD^{(k)}|}\sum_{m = 1}^{|\calD^{(k)}|}||d^{(m)} - \breve d^{(m)}||\ge c_a\cdot\frac{1}{|\calD^{(k)}|}\sum_{m = 1}^{|\calD^{(k)}|}||g(d^{(m)}) - g(\breve d^{(m)})||$ and $\frac{1}{|\calD^{(k)}|}\sum_{m = 1}^{|\calD^{(k)}|}||d_i^{(m)} - d^{(m)}||\le \frac{1}{|\calD^{(k)}|}\sum_{m = 1}^{|\calD^{(k)}|} c_b ||g(d_i^{(m)}) - g(d^{(m)})||$.

We denote $g^{(k)}$ as the original gradient of client $k$, and $\wtilde{g}^{(k)}$ as the distorted gradient uploaded from client $k$ to the server. Therefore, we have that
\begin{align} 
   \|\frac{1}{|\calD^{(k)}|}\sum_{m = 1}^{|\calD^{(k)}|} (g(d^{(m)}) - g(\breve d^{(m)}))\| & = \| \wtilde{g}^{(k)} - g^{(k)}\|\\
   & = \|\frac{1}{|\calD^{(k)}|}\sum_{m = 1}^{|\calD^{(k)}|}(g(\breve d^{(m)}) + \delta^{(k)} - g(\breve d^{(m)}))\|\\
   & = \|\delta^{(k)}\|\\
   & = \Delta^{(k)}.
\end{align}

Therefore, we have
\begin{align*}
    D(1-\epsilon_p^{(k)}) = \frac{1}{I}\sum_{i = 1}^{I} \frac{1}{|\calD^{(k)}|}\sum_{m = 1}^{|\calD^{(k)}|}||d_i^{(m)} - \breve d^{(m)}||
    &\ge  c_a\Delta^{(k)} - c_b\cdot\frac{1}{I}\cdot\frac{1}{|\calD^{(k)}|}\sum_{i = 1}^{I}\sum_{m = 1}^{|\calD^{(k)}|}||g(d_i^{(m)}) - g(d^{(m)})||\\
    &\ge c_a\Delta^{(k)} - c_2\cdot c_b I^{p-1}\\
    &\ge\frac{1}{2}\max\{c_a\Delta^{(k)}, c_2\cdot c_b I^{p-1}\} \\
    &\ge \frac{c_a\Delta^{(k)} + c_2\cdot c_b I^{p-1}}{4}.
\end{align*}
Therefore, we have that
\begin{align*}
    \epsilon_{p}^{(k)} \le 1 - \frac{c_a\Delta^{(k)} + c_2\cdot c_b I^{p-1}}{4D}.
\end{align*}
Therefore, we have that
\begin{align}
   \epsilon_{p}^{(k)} &\le 1 - \frac{c_a\Delta^{(k)} + c_2\cdot c_b I^{p-1}}{4D}\\
    & = 1 - \frac{c_a\|\frac{1}{|\calD^{(k)}|}\sum_{m = 1}^{|\calD^{(k)}|} (g(d^{(m)}) - g(\breve d^{(m)}))\| + c_2\cdot c_b I^{p-1}}{4D}.
\end{align}
\end{proof}

\section{Analysis for \pref{thm: first_main_result}}\label{app: first_main_result}
\begin{thm}[Relationship between Privacy and Robustness]\label{thm: first_main_result_app}   
    Let the privacy leakage of the protection mechanism $\calA$ be measured using \pref{defi: data_privacy}, and let the robustness of $\calA$ be measured using \pref{defi: input_robustness}. Let the privacy leakage be $\epsilon_p$, then $\left(r, \frac{Cr}{2} + \frac{4D(1-\epsilon_p) - c_2 c_b I^{p-1}}{2c_a}\right)$, where $c_a$ and $D$ represent constants, $\RcalD$ represents the dataset, and $\gamma$ represents the error rate.
\end{thm}
\begin{proof}
The distortion extent measures the discrepancy between the gradient of the original data and the distorted data. From the relationship between distortion extent and privacy leakage (\pref{lem: bound_for_privacy_leakage_mt}), we know that the distortion extent is bounded as follows:

\begin{align}\label{eq: distortion_bound}
    \|\frac{1}{|\calD^{(k)}|}\sum_{m = 1}^{|\calD^{(k)}|} (g(d^{(m)}) - g(\breve d^{(m)}))\|\le\frac{4D(1-\epsilon_p) - c_2 c_b I^{p-1}}{c_a}.
\end{align}
On one hand, we have that
\begin{align}
    \| \nabla \calL(\theta, (X + \delta, y)) - \nabla \calL(\theta, (X, y))\| 
    & = \|\frac{1}{|\calD^{(k)}|}\sum_{m = 1}^{|\calD^{(k)}|} (g(d^{(m)}) - g(\breve d^{(m)}))\|\\
    &\le\frac{4D(1-\epsilon_p) - c_2 c_b I^{p-1}}{c_a}.
\end{align}
On the other hand, from Lipschitz's condition, we know that
\begin{align}
    \|\nabla\calL(\theta, (X + \delta, y)) - \nabla\calL(\theta, (X, y))\|\le C\|\delta\|,
\end{align}
where $C$ represents a constant.
Therefore, we have that
\begin{align}
    \|\nabla\calL(\theta, (X + \delta, y)) - \nabla\calL(\theta, (X, y))\|\le \frac{C\|\delta\|}{2} + \frac{4D(1-\epsilon_p) - c_2 c_b I^{p-1}}{2c_a}.
\end{align}
Therefore, we have that
   \begin{align} 
       &\E_{X\sim P}[\text{sup}_{\|\delta\|_p\le r} |\nabla\calL(\theta, (X + \delta, y)) - \nabla\calL(\theta, (X, y))|]\\
       &\le\frac{C\|\delta\|}{2} + \frac{4D(1-\epsilon_p) - c_2 c_b I^{p-1}}{2c_a}\\
       &\le\frac{Cr}{2} + \frac{4D(1-\epsilon_p) - c_2 c_b I^{p-1}}{2c_a}. 
   \end{align}
   
Recall that a model is $(r, \epsilon)$-input-robust, if it holds that
\begin{align}
       \E_{X\sim P}[\text{sup}_{\|\delta\|_p\le r} |\nabla\calL\left(w, (X + \delta, y)\right) - \nabla\calL\left(w, (X,y) \right)|]\le\epsilon,
   \end{align}
where $X$ denotes the data (training sample or feature), and $y$ denotes the corresponding label.

Therefore, from the definition of robustness, we know that the model is $\left(r, \frac{Cr}{2} + \frac{4D(1-\epsilon_p) - c_2 c_b I^{p-1}}{2c_a}\right)$-input-robust.
\end{proof}

\end{document}